\documentclass[letterpaper, 10 pt, conference]{ieeeconf}
\IEEEoverridecommandlockouts    
\usepackage{cite}
\usepackage{amsmath,amssymb,mathrsfs,amsfonts,bbm}
\usepackage{color}
\usepackage{algorithmic}
\usepackage{graphicx}
\usepackage{subcaption}
\usepackage{textcomp}
\usepackage{mathtools}

\newtheorem{problem}{Problem}
\newtheorem{definition}{Definition}

\newtheorem{theorem}{Theorem}

\begin{document}

\title{\LARGE \bf Multi-Robot-Assisted Human Crowd Evacuation using Navigation Velocity Fields}

\author{Tongjia Zheng$^{1}$, Zhenyuan Yuan$^{2}$, Mollik Nayyar$^{3}$, Alan R. Wagner$^{3}$, Minghui Zhu$^{2}$, and Hai Lin$^{1}$
\thanks{*This work was supported by the National Science Foundation under Grant No. CNS-1830335 and CNS-1830390. Any opinions, findings, and conclusions or recommendations in this material are those of the author(s) and do not necessarily reflect the views of the National Science Foundation.}
\thanks{$^{1}$Tongia Zheng and Hai Lin are with the Department of Electrical Engineering, University of Notre Dame, Notre Dame, IN 46556, USA. {\tt\small tzheng1@nd.edu, hlin1@nd.edu.}} 
\thanks{$^{2}$Zhenyuan Yuan and Minghui Zhu are with the Department of Electrical Engineering,  The Pennsylvania State University, University Park, PA 16802, USA. {\tt\small zqy5086@psu.edu, muz16@psu.edu.}} 
\thanks{$^{3}$Mollik Nayyar and Alan R. Wagner are with the Department of Aerospace Engineering, The Pennsylvania State University, University Park, PA 16802, USA. {\tt\small mxn244@psu.edu, azw78@psu.edu.}}
}

\maketitle

\begin{abstract}
This work studies a robot-assisted crowd evacuation problem where we control a small group of robots to guide a large human crowd to safe locations.
The challenge lies in how to model human-robot interactions and design robot controls to indirectly control a human population that significantly outnumbers the robots.
To address the challenge, we treat the crowd as a continuum and formulate the evacuation objective as driving the crowd density to target locations.
We propose a novel mean-field model which consists of a family of microscopic equations that explicitly model how human motions are locally guided by the robots and an associated macroscopic equation that describes how the crowd density is controlled by the navigation velocity fields generated by all robots.
Then, we design density feedback controllers for the robots to dynamically adjust their states such that the generated navigation velocity fields drive the crowd density to a target density.
Stability guarantees of the proposed controllers are proven.
Agent-based simulations are included to evaluate the proposed evacuation algorithms.
\end{abstract}


\section{Introduction}

Emergency evacuation can be a chaotic situation defined by the need to relocate a possibly large number of people to safety without causing choke points that slow the evacuation process \cite{bryner2007reconstructing, robinette2012information}. 
Robots may be able to help by guiding people away from choke points, encouraging them to evacuate, and providing situation awareness \cite{wagner2021robot}. 
A number of critical technologies will need to be in place prior to the deployment and use of robots as emergency evacuation guides. 
This paper considers the challenge of deploying a small number of robots to guide the evacuation of a large number of people. 

The challenge of developing emergency guidance robots has primarily considered only single robot, single human evacuation scenarios \cite{shell2005insights, robinette2016investigating, NayyarWagner2019}, which typically assume a simplified environment with little or no interference from other evacuees or robots. In contrast, a multi-robot approach to evacuating large numbers of people has several important advantages. First, distributed robots could identify and help evacuate more people more quickly. Second, distributed robots may be more robust in their ability to locate evacuees and find exits. Finally, distributed robots could cover a larger evacuation area. For these reasons the development of approaches that utilize multiple robots is vital.

Researchers have investigated multi-robot-assisted evacuations by integrating the robots' impact into the existing crowd models, including microscopic and macroscopic models \cite{sakour2017robot}.
Microscopic models include agent-based models \cite{luo2008agent} and cellular automata \cite{boukas2014robot}.
They are similar in defining a set of microscopic interaction rules such that certain group behaviors emerge, but have different trade-offs regarding the level of detail and computational cost.
Macroscopic models, on the other hand, treat the crowd as a continuum flow \cite{treuille2006continuum}.
Some of these models have been extended to involve human-robot interactions.
For example, Boukas et al. used cellular automata to simulate the movement of a crowd guided by robots \cite{boukas2014robot}.
Okada et al. used swarm control methods and demonstrate the use of guidance robots to evacuate other robots based on vector fields \cite{okada2011optimization}.
In general, microscopic crowd control has difficulties in guaranteeing evacuation performance, while macroscopic crowd control lacks details for individual movements.
Other methodologies have also been pursued.
For example, Zhang and Guo presented a multi-agent cooperative seeking algorithm to guide evacuees while maintaining a predefined formation during movement \cite{zhang2015distributed}.
Yet, forcing evacuees into a robot-mediated formation limits its applicability to small groups.

Given the prior work, we follow a mixed micro-macroscopic approach and adopt a mean-field model for human crowds.
Mean-field models have been recognized as more promising models for crowd dynamics \cite{borsche2018numerical}, which also treat the crowd as a continuum and use the density of individual states to represent the group state.
Their advantage over the others is that they consist of not only microscopic equations for individual motions and interactions, but also an associated macroscopic equation to describe the density dynamics.
Mean-field models emphasize that the microscopic and macroscopic descriptions are equivalent in the mean-field limit (as the population tends to infinity) \cite{bensoussan2013mean}.
Hence, they can easily study the influence of group dynamics on individual behaviors, which is difficult for other models.

Using mean-field models, the evacuation objective can be naturally formulated as a density control problem of driving the crowd density to target locations.
Density control has been used to control other types of large-scale systems such as robotic swarms \cite{elamvazhuthi2019mean,zheng2020complex, zheng2021transporting}.
In particular, density feedback is widely used to design closed-loop controllers.
Unlike robotic systems that are directly controllable, humans can only be indirectly influenced by human-robot interactions.
Because of this unique challenge, mean-field models have mainly served as tools for predicting crowd behaviors and remain largely unexplored as a tool for robot-assisted human crowd evacuation. 
An exception is the work in \cite{liu2018coordinated} where multiple robots are deployed in a crowd to guide their density evolution.
Nevertheless, the robots are assumed to be static and therefore this control strategy is essentially open-loop.

{\bf Contribution statement:} 
In this paper, we extend the mean-field models by explicitly integrating human-robot interactions and design robot controls based on these models.
Specifically, the robots' impact on humans is modeled through navigation velocity fields which are determined by the robots' states (like their positions).
We exploit the density feedback technique and backstepping design to derive closed-loop robot controls.
With these controllers, the robots dynamically adjust their states such that the generated navigation velocity fields drive the crowd density to a target density.
Exponential convergence with sufficient robots and bounded stability in the lack of robots are both proven.
The proposed algorithm is evaluated using a set of agent-based simulations.


\section{Preliminaries}
\label{section:preliminaries}

\subsection{Notations}\label{section:notation}
The Euclidean norm of $x\in\mathbb{R}^n$ is denoted by $\|x\|$.
Let $E\subset\mathbb{R}^n$ be a measurable set. 
For $f:E\to\mathbb{R}$, its $L^2$-norm is denoted by $\|f\|_{L^2(E)}:=(\int_{E}|f(x)|^{2}dx)^{1/2}$.
We omit $E$ in the notation when it is clear.
The gradient and Laplacian of a scalar function $f$ are denoted by $\nabla f$ and $\Delta f$, respectively.
The divergence of a vector field $F$ is denoted by $\nabla\cdot F$. 

\subsection{Input-to-state stability}
Input-to-state stability (ISS) is a stability notion to study nonlinear systems with external inputs.
We introduce its extension to infinite-dimensional systems \cite{dashkovskiy2013input}.
Define
\begin{align*}
    \mathcal{P} &:=\{\gamma:\mathbb{R}_+\to\mathbb{R}_+|\gamma\text{ is continuous, }\gamma(0)=0,\\
    &\qquad\text{ and }\gamma(r)>0\text{ for }r>0\}\\
    \mathcal{K} &:=\{\gamma\in\mathcal{P}\mid\gamma\text{ is strictly increasing}\} \\
    \mathcal{K}_\infty &:=\{\gamma\in\mathcal{K}\mid\gamma\text{ is unbounded}\} \\
    \mathcal{L} &:=\{\gamma:\mathbb{R}_+\to\mathbb{R}_+\mid\gamma\text{ is continuous and strictly} \\
    &\quad\quad\text{decreasing with }\lim_{t\to\infty}\gamma(t)=0\} \\
    \mathcal{KL} &:=\{\beta:\mathbb{R}_+\times\mathbb{R}_+\to\mathbb{R}_+\mid\beta(\cdot,t)\in\mathcal{K},\forall t\geq0,\\
    &\qquad\beta(r,\cdot)\in\mathcal{L},\forall r>0\}.
\end{align*}

Let $\left(X,\|\cdot\|_X\right)$ and $\left(U,\|\cdot\|_{U}\right)$ be the state and input space, endowed with norms $\|\cdot\|_X$ and $\|\cdot\|_{U}$, respectively.
Denote by $U_c=PC(\mathbb{R}_+;U)$ the space of piecewise continuous functions from $\mathbb{R}_+$ to $U$, equipped with the sup-norm.
Consider a control system $\Sigma=(X,U_c,\phi)$ where $\phi: \mathbb{R}_+\times X\times U_c\to X$ is a transition map.
Let $x(t)=\phi(t,x_0,u)$.

\begin{definition} \label{dfn:(L)ISS}
$\Sigma$ is called \textit{input-to-state stable (ISS)}, if $\exists\beta\in\mathcal{KL},\gamma\in\mathcal{K}$, such that
\begin{align*}
    \|x(t)\|_X\leq\beta(\|x_0\|_X, t)+\gamma\Big(\sup_{0\leq s\leq t}\|u(s)\|_{U}\Big),
\end{align*}
$\forall x_0\in X,\forall u\in U_c$ and $\forall t\geq0$.
\end{definition} 


\begin{definition}\label{dfn:(L)ISS-Lyapunov function}
A continuous function $V:\mathbb{R}_+\times X\to\mathbb{R}_+$ is called an \textit{ISS-Lyapunov function} for $\Sigma$, if $\exists\psi_{1},\psi_{2}\in\mathcal{K}_\infty,\chi\in\mathcal{K}$, and $W\in\mathcal{P}$, such that:
\begin{itemize}
    \item[(i)] $\psi_1(\|x\|_X)\leq V(t,x)\leq\psi_2(\|x\|_X), ~\forall x\in X$
    \item[(ii)] $\forall x\in X,\forall u\in U_c$ with $u(0)=\xi\in U$ it holds:
    \begin{align*}
        \|x\|_X\geq\chi(\|\xi\|_U) \Rightarrow \dot{V}(t,x)\leq-W(\|x\|_X).
    \end{align*}
\end{itemize}
\end{definition}


\begin{theorem}\label{thm:ISS-Lyapunov function}
If $\Sigma$ admits an ISS-Lyapunov function, then it is ISS.
\end{theorem}

\section{Problem formulation}
\label{section:problem formulation}

In this work, we study the problem of controlling a small number of robots to direct a large human crowd.
This is achieved by controlling the robots to dynamically generate navigation velocity fields to indirectly control the evolution of the crowd density.
We start with modeling the behaviors of humans, robots, and their interactions.

Denote by $\Omega\subset\mathbb{R}^2$ the workspace, which is assumed to be a convex bounded domain with a Lipschitz boundary $\partial\Omega$.
Let $N$ be the number of humans.
For the $j$-th human, denote by $X_j(t)\in\Omega$ its position.
Assume the human motions are affected by other humans and the robots according to:
\begin{align}
    dX_j= & \Big(\frac{1}{N}\sum_{k=1}^N\nabla W(X_k-X_j)+v_r(X_j,t)+v_a(X_j,t)\Big)dt \nonumber\\
    & +\sqrt{2\sigma(X_j,t)}dB_j(t), \quad j=1,\dots,N. \label{eq:human model}
\end{align}

In this model, $W$ is a pairwise interaction potential between humans (which can model, e.g., the repulsive behavior \cite{helbing2002simulation}).
$\{B_j\}_{j=1}^N$ are a set of independent Wiener processes that represent random motions with standard deviation $\sqrt{2\sigma(X_j,t)}$.
$W$ and $\sigma$ are assumed to be known and are optional.
$v_a$ represents the automatically generated additional velocities to avoid collision with the robots.
Since $v_a$ is non-deterministic and resumes $0$ when the avoidance behavior terminates, it is time-varying, unknown, but bounded.

Now we discuss how to model $v_r$, the robot-guided velocities.
In our setup, every robot affects the human motions by providing a directed sign, such as an arrow with instructions ``THIS WAY''.
Intuitively, humans tend to follow the sign when they are close to the robots and have less motivation to follow it when they are further away.
These behaviors have been confirmed by experiments in \cite{robinette2014assessment}.
Therefore, the navigation velocity field can be modeled as follows.

Let $n$ be the number of robots ($n\ll N$).
For the $i$-th robot, denote by $R_i(t)\in\Omega$ its position and $\theta_i(t)\in[0,2\pi)$ the orientation of its sign.
The robot models are given by:
\begin{align}\label{eq:robot model}
\begin{split}
    \dot{R}_i(t) & =u_i(t),\quad i=1,\dots,n \\
    \dot{\theta}_i(t) & =\omega_i(t),
\end{split}
\end{align}
where $u_i(t)\in\mathbb{R}^2$ and $\omega_i(t)\in\mathbb{R}$ are the control inputs.
Its generated navigation velocity field is modeled by:
\begin{align*}
    v_{r_i}(x,t) & =K(x-R_i(t),\theta_i(t))=\begin{bmatrix}
    \bar{K}(x-R_i(t))\cos\theta_i(t) \\
    \bar{K}(x-R_i(t))\sin\theta_i(t)
    \end{bmatrix}
\end{align*}
where $\bar{K}(\xi)\in\mathbb{R}_{>0}$ for $\xi\in\mathbb{R}^2$ is a continuously differentiable kernel function which decays with $\|\xi\|$.
$\bar{K}(\xi)$ can be designed based on experience or learned from experiment data.
According to the experiment result in \cite{robinette2014assessment}, the following design is adopted in this work:
\begin{align*}
    \bar{K}(\xi)=c\exp(-\xi^2/\eta),
\end{align*}
where $c$ and $\eta$ are positive constants that represent its magnitude and range of influence, respectively.

Intuitively, $v_{r_i}$ is a velocity field whose direction is specified by $[\cos\theta_i(t)~\sin\theta_i(t)]^T$ and whose magnitude achieves its maximum at $R_i(t)$ and decays when the distance from $R_i(t)$ increases.
The collective navigation velocity field generated by all robots is thus given by:
\begin{align}\label{eq:navigation velocity}
    v_r(x,t)=\sum_{i=1}^nv_{r_i}(x,t)=\sum_{i=1}^nK(x-R_i(t),\theta_i(t)).
\end{align}

To address the challenge of the large human population, we consider the mean-field limit as $N\to\infty$.
In this way, the crowd behavior can be captured by its probability density
\begin{align*}
    \rho(x,t)\approx\frac{1}{N}\sum_{j=1}^N\delta_{X_j(t)}
\end{align*}
with $\delta_{x}$ being the Dirac distribution, and it satisfies the following mean-field model \cite{bensoussan2013mean}, a Fokker-Planck equation with a zero-flux boundary to confine the crowd within $\Omega$:
\begin{align}\label{eq:density dynamics}
\begin{split}
    \partial_t\rho=-\nabla\cdot\big((\nabla W*\rho+v_r+v_a)\rho\big)+\Delta(\sigma\rho) & \text{ on } \Omega \\
    \big((\nabla W*\rho+v_r)\rho-\nabla(\sigma\rho)\big)\cdot\boldsymbol{n}=0 & \text{ on } \partial\Omega
\end{split}
\end{align}
where $*$ denotes convolution and $\boldsymbol{n}$ is the unit inner normal to the boundary $\partial\Omega$.
This equation shows how the crowd density $\rho$ is controlled by the navigation velocity field $v_r$ generated by the robots.
The crowd evacuation problem is then formulated as a density control problem as follows.

\begin{problem}[Crowd evacuation]
Given the dynamics of crowd density \eqref{eq:density dynamics} and a target density $\rho_*(x)$, design control laws $\{u_i(t),\omega_i(t)\}_{i=1}^{n}$ for the robots to drive the crowd density $\rho(x,t)$ towards $\rho_*(x)$.
\end{problem}

\section{Control design}
\label{section:control design}

In this section, we design feedback controllers for the robots.
It will be shown that ``feedback'' has twofold implications as the controllers take not only state feedback from the robots themselves but also density feedback from the crowd.

First of all, we show how the crowd density is eventually controlled by the robots' inputs $\{u_i(t),\omega_i(t)\}_{i=1}^{n}$.
Denote
\begin{align*}
    & K_\xi(\xi,\theta):=\frac{\partial}{\partial\xi}K(\xi,\theta)\in\mathbb{R}^{2\times2} \\
    & K_\theta(\xi,\theta):=\frac{\partial}{\partial\theta}K(\xi,\theta)\in\mathbb{R}^2,
\end{align*}
and for the $i$-th robot herder, denote
\begin{align*}
    & K_\xi^i(x,t)=K_\xi(x-R_i(t),\theta_i(t)), \\
    & K_\theta^i(x,t)=K_\theta(x-R_i(t),\theta_i(t)).
\end{align*}
Taking the time derivative of $v_r(x,t)$, we obtain:
\begin{align}\label{eq:navigation velocity dynamics}
\begin{split}
    \partial_tv_r(x,t) & =\sum_{i=1}^n\partial_tK(x-R_i(t),\theta_i(t)) \\
    & =\sum_{i=1}^n\Big(-K_\xi^i(x,t)u_i(t)+K_\theta^i(x,t)\omega_i(t)\Big),
\end{split}
\end{align}
which shows how the navigation velocity field is controlled by the robot inputs.

Equations \eqref{eq:density dynamics} and \eqref{eq:navigation velocity dynamics} constitute an infinite-dimensional control problem for crowd evacuation. 
We notice that it has a cascade structure, i.e., $v_r$, the state of \eqref{eq:navigation velocity dynamics}, acts as a virtual input of \eqref{eq:density dynamics}.
This motivates us to adopt a backstepping design that recursively constructs a sequence of stabilizing functions to stabilize systems with cascade structures \cite{krstic1995nonlinear}.

\subsection{Backstepping design: step I}
In step I, we design a virtual stabilizing control input for \eqref{eq:density dynamics}.
Define $\tilde{\rho}(x,t)=\rho(x,t)-\rho_*(x)$ and design
\begin{align}\label{eq:vd}
    v_d=-\frac{\alpha\nabla\tilde{\rho}-\nabla(\sigma\rho)}{\rho}-\nabla W*\rho,
\end{align}
where $\alpha(x,t)>0$ is an adjustable control gain.
The fact that $\alpha$ can be a function of both $x$ and $t$ is very useful in practice because it allows us to assign different velocity magnitudes at different locations depending on, e.g., the current density $\rho$ at that location.
We have the following theorem.

\begin{theorem}\label{thm:step I convergence}
Consider \eqref{eq:density dynamics}.
If $v_a\equiv0$ and $v_r\equiv v_d$, then $\|\tilde{\rho}(x,t)\|_{L^2(\Omega)}\to0$ exponentially and $V_1(t)=\int_\Omega\tilde{\rho}(x,t)^2dx$ is a Lyapunov certificate that satisfies $\frac{dV_1}{dt}\leq\int_\Omega -k_p\tilde{\rho}^2dx$ for some constant $k_p>0$.
\end{theorem}

\begin{proof}
Substituting \eqref{eq:vd} into \eqref{eq:density dynamics}, we obtain
\begin{align*}
    \partial_t\tilde{\rho}=\nabla\cdot(\alpha\nabla\tilde{\rho}) & \quad \text{on} \quad \Omega \\
    \nabla\tilde{\rho}\cdot\boldsymbol{n}=0 & \quad \text{on} \quad \partial\Omega
\end{align*}
which is a diffusion equation with the Neumann boundary condition.
Proof can be found in \cite{zheng2021transporting}.
\end{proof}

Control laws like \eqref{eq:vd} are called density feedback as they explicitly depend on the density \cite{elamvazhuthi2019mean, zheng2020complex, zheng2021transporting}.
The density can be estimated using classical density estimation algorithms like kernel density estimation (KDE) \cite{silverman1986density} or the density filtering algorithm in \cite{zheng2020pde, zheng2021distributedmean}.
For simplicity, we adopt KDE in this work, which constructs a density estimate $\hat{\rho}$ based on the human positions $\{X_j(t)\}_{j=1}^N$ according to:
\begin{equation*}
    \hat{\rho}(x,t) = \frac{1}{N h^2} \sum_{j=1}^{N} H\left(\frac{1}{h}\left(x-X_{j}(t)\right)\right),
\end{equation*}
where $H(x)$ is a kernel function and $h$ is the bandwidth.
We choose the Gaussian kernel:
$$
H(x)=\frac{1}{2\pi}\exp\big(-\frac{1}{2}x^Tx\big).
$$


Before continuing to design the robot control, we present a stability result that ensures that $\|\tilde{\rho}(x,t)\|_{L^2(\Omega)}$ remains bounded even if the robots can never accurately generate the desired velocity field $v_d$ (which can happen when there are no enough robots) and if the avoidance velocity $v_a$ presents.
Define $\tilde{v}(x,t)=v_r(x,t)-v_d(x,t)$.

\begin{theorem}\label{thm:step I ISS}
Consider \eqref{eq:density dynamics}.
Assume $\sup_{x,t}{\rho}<\infty$.
Then $\|\tilde{\rho}(x,t)\|_{L^2(\Omega)}$ is ISS to $\|\tilde{v}(x,t)+v_a(x,t)\|_{L^2(\Omega)}$.
\end{theorem}

\begin{proof}
Substituting $v_r=v_d+\tilde{v}$ into \eqref{eq:density dynamics}, we obtain
\begin{align*}
    \partial_t\tilde{\rho}=\nabla\cdot(\alpha\nabla\tilde{\rho}-\tilde{v}\rho-v_a\rho) & \quad \text{on} \quad \Omega \\
    (\alpha\nabla\tilde{\rho}-\tilde{v}\rho-v_a\rho)\cdot\boldsymbol{n}=0 & \quad \text{on} \quad \partial\Omega.
\end{align*}
By the divergence theorem, Poincar\'e's inequality, and the fact that $\int_\Omega\tilde{\rho}dx=0$, we have
\begin{align*}
    \frac{dV_1}{dt} & =\int_\Omega-\alpha(\nabla\tilde{\rho})^2+\rho\nabla\tilde{\rho}\cdot(\tilde{v}+v_a)dx\\
    & \leq\int_\Omega-\alpha(\nabla\tilde{\rho})^2+\|\nabla\tilde{\rho}\|\|\rho(\tilde{v}+v_a)\|dx\\
    & \leq-(1-c_1)\alpha_{\inf}\|\nabla\tilde{\rho}\|_{L^2}^2-c_1\alpha_{\inf}\|\nabla\tilde{\rho}\|_{L^2}^2 \\
    & \quad+\|\nabla\tilde{\rho}\|_{L^2}\|\rho(\tilde{v}+v_a)\|_{L^2}\\
    & \leq-(1-c_1)c_2^2\alpha_{\inf}\|\tilde{\rho}\|_{L^2}^2-c_1c_2\alpha_{\inf}\|\nabla\tilde{\rho}\|_{L^2}\|\tilde{\rho}\|_{L^2}\\
    & \quad+\|\nabla\tilde{\rho}\|_{L^2}\|\rho(\tilde{v}+v_a)\|_{L^2},
\end{align*}
where $\alpha_{\inf}=\inf_{(x,t)}\alpha(x,t)$, $c_1\in(0,1)$, and $c_2>0$ is the constant from Poincar\'e's inequality.
Since $\sup_{x,t}{\rho}<\infty$, there exists a positive constant $c_3$ such that $\|\rho(\tilde{v}+v_a)\|_{L^2}\leq c_3\|\tilde{v}+v_a\|_{L^2}$.
Hence, if
\begin{align*}
    \|\tilde{\rho}\|_{L^2}\geq\frac{c_3}{c_1c_2\alpha_{\inf}}\|\tilde{v}+v_a\|_{L^2},
\end{align*}
then we have
\begin{align*}
    \frac{dV_1}{dt}\leq-(1-c_1)c_2^2\alpha_{\inf}\|\tilde{\rho}\|_{L^2}^2.
\end{align*}
By Theorem \ref{thm:ISS-Lyapunov function}, $\|\tilde{\rho}\|_{L^2(\Omega)}$ is ISS to $\|\tilde{v}+v_a\|_{L^2(\Omega)}$.
\end{proof}

This theorem ensures that the density convergence error $\|\tilde{\rho}\|_{L^2(\Omega)}$ always remains bounded by a positive function of $\|\tilde{v}+v_a\|_{L^2(\Omega)}$ and converges to 0 if $\|\tilde{v}+v_a\|_{L^2(\Omega)}$ vanishes.

\subsection{Backstepping design: step II}
In step II, we design $\{u_i(t),\omega_i(t)\}_{i=1}^{n}$ for the robots such that $(\tilde{\rho},\tilde{v})\to0$ based on the results in step I.
We assume $v_a\equiv0$ since its effect can be studied in a similar way as in Theorem \ref{thm:step I ISS}.
Consider an augmented Lyapunov function:
\begin{align}\label{eq:augmented Lyapunov}
    V_2(t)=\int_\Omega\frac{1}{2}\tilde{\rho}(x,t)^2+\frac{1}{2}\tilde{v}(x,t)^T\tilde{v}(x,t)dx.
\end{align}
By the divergence theorem and Theorem \ref{thm:step I convergence}, we have 
\begin{align*}
    \frac{dV_2}{dt}= & \int_\Omega\tilde{\rho}\nabla\cdot(\alpha\nabla\tilde{\rho}-\tilde{v}\rho)+\tilde{v}^T(\partial_tv_r-\partial_tv_d)dx \\
    = & \int_\Omega-\alpha\|\nabla\tilde{\rho}\|^2+\tilde{v}^T(\rho\nabla\tilde{\rho}+\partial_tv_r-\partial_tv_d)dx \\
    \leq & \int_\Omega-k_p\tilde{\rho}^2+\tilde{v}^T(\rho\nabla\tilde{\rho}+\partial_tv_r-\partial_tv_d)dx \\
    = & -\int_\Omega k_p\tilde{\rho}^2dx+\int_\Omega\tilde{v}^T(\rho\nabla\tilde{\rho}-\partial_tv_d)dx \\
    & +\sum_{i=1}^n\Big(-\int_\Omega\tilde{v}^TK_\xi^iu_idx+\int_\Omega\tilde{v}^TK_\theta^i\omega_idx\Big).
\end{align*}

The remaining problem is to design $\{u_i(t),\omega_i(t)\}_{i=1}^{n}$ such that $\frac{dV_2}{dt}$ becomes negative definite.
Denote by $x^{(l)}$ the $l$-th element of vector $x$.
Then we pick
\begin{align}\label{eq:u and omega}
\begin{split}
    u_i^{(l)} & =k_{ui}\int_\Omega(\tilde{v}^TK_\xi^i)^{(l)}dx+\frac{\beta_{il}\int_\Omega\tilde{v}^T(\rho\nabla\tilde{\rho}-\partial_tv_d)dx}{\int_\Omega(\tilde{v}^TK_\xi^i)^{(l)}dx} \\
    \omega_i & =-k_{\omega i}\int_\Omega\tilde{v}^TK_\theta^idx-\frac{\gamma_{i}\int_\Omega\tilde{v}^T(\rho\nabla\tilde{\rho}-\partial_tv_d)dx}{\int_\Omega\tilde{v}^TK_\theta^idx},
\end{split}
\end{align}
where the scalars $\{k_{ui}(t),k_{\omega i}(t)\}$ satisfy
\begin{align*}
    \sum_{i=1}^nk_{ui}(t)>0 \text{ and } \sum_{i=1}^nk_{\omega i}(t)>0, \forall t
\end{align*}
and the scalars $\{\beta_{il}(t),\gamma_{i}(t)\}$ satisfy
\begin{align}\label{eq:control parameter constraints}
\begin{cases}
    \beta_{il}=0, \text{ if } \int_\Omega(\tilde{v}^TK_\xi^i)^{(l)}dx=0 \\
    \gamma_i=0, \text{ if } \int_\Omega\tilde{v}^TK_\theta^idx=0\\
    \sum_{i,l}\beta_{il}(t)+\gamma_{i}(t)=1, \forall t.
\end{cases}
\end{align}
Since the columns of $K_\xi^i$ and $K_\theta^i$ are linearly independent, at least one of the two integrals $\int_\Omega(\tilde{v}^TK_\xi^i)^{(l)}dx$ and $\int_\Omega\tilde{v}^TK_\theta^idx$ would be nonzero unless $\tilde{v}=0$.
In other words, the conditions \eqref{eq:control parameter constraints} always have a solution as long as $\tilde{v}\neq0$.
Substituting \eqref{eq:u and omega} into $\frac{dV_2}{dt}$, we obtain
\begin{align*}
    \frac{dV_2}{dt}\leq\int_\Omega-k_p\tilde{\rho}^2-\sum_{i=1}^n\big(k_{ui}\|\tilde{v}^TK_\xi^i\|^2+k_{\omega i}\|\tilde{v}^TK_\theta^i\|^2\big)dx.
\end{align*}
We thus obtain the following convergence result.

\begin{theorem}\label{thm:step II convergence}
Consider the complete system \eqref{eq:density dynamics} and \eqref{eq:navigation velocity dynamics}. 
Assume $v_a\equiv0$.
Let $\{u_i(t),\omega_i(t)\}_{i=1}^{n}$ be given by \eqref{eq:u and omega} where $v_d$ is given by \eqref{eq:vd}.
Then $(\|\tilde{\rho}(x,t)\|_{L^2(\Omega)},\|\tilde{v}(x,t)\|_{L^2(\Omega)})\to0$ exponentially and \eqref{eq:augmented Lyapunov} is a Lyapunov certificate.
\end{theorem}

As mentioned at the beginning, the robot control laws \eqref{eq:u and omega} are simultaneously density feedback and state feedback because they explicitly depend on the real-time density $\rho$ and the robot states $\{R_i,\theta_i\}_{i=1}^n$.
Using the crowd density to represent the crowd behavior addresses the scalability issue of large human populations, and the backstepping design based on density feedback ensures closed-loop stability and robustness to disturbances such as the avoidance velocity $v_a$.

\subsection{Discussion}

We discuss some practical issues related to the number of robots.
Theoretically, even one robot is enough to stabilize the crowd density.
This is because we allow arbitrarily large values for $\{u_i,\omega_i\}$.
To facilitate the discussion, for a vector-valued function $f(x)$ on $\Omega$ and a small constant $\epsilon>0$, define
\begin{align*}
    M_\epsilon(f)=\{x\in\Omega\mid \|f(x)\|>\epsilon\},
\end{align*}
which represents the set where the mass of $f(x)$ concentrates.
Take $u_i$ as an example.
If $M_\epsilon(\tilde{v})\cap M_\epsilon(K_\xi^i)=\emptyset$, then $u_i$ would be very large according to \eqref{eq:u and omega}.
This means that the velocity error concentrates on some regions where there are no robots.
In this case, the control law \eqref{eq:u and omega} may keep generating large velocities $\{u_i\}$ that instantly drive the robots to new positions.
To avoid large velocities, we can impose more constraints on the selection of $\{\beta_{il},\gamma_{i}\}$, such as requiring $\beta_{il}=0$ if $\int_\Omega(\tilde{v}^TK_\xi^i)^{(l)}dx<\epsilon$.
In this case, however, \eqref{eq:control parameter constraints} may not have a solution and $\tilde{v}$ may not converge to 0.
This is essentially caused by the lack of robots or under-actuation.
Nevertheless, we still ensure that the density error $\|\tilde{\rho}(x,t)\|_{L^2(\Omega)}$ remains bounded according to Theorem \ref{thm:step I ISS}.

The minimum number of robots depends on many factors including $\rho_*$, $\bar{K}$, $v_d$ and the bounds for $\{u_i,\omega_i\}$, which is a subject of continuing work. 
A rough lower estimate for $n$ is that when the robots are regularly distributed in $\Omega$, we have
\begin{align*}
    \bigcup_{i=1}^nM_\epsilon(v_{r_i})=\Omega,
\end{align*}
for a sufficiently large $\epsilon$, i.e., their range of influence can cover the whole workspace.

\section{Agent-based simulations}
\label{section:simulation}

In this section, we present two agent-based simulation studies to validate the control algorithms presented above.

The first simulation is based on MATLAB.
All humans and robots are treated as mass points whose motions are simulated according to \eqref{eq:human model} and \eqref{eq:robot model} using Euler's method ($dt=0.01$).
The robots generate navigation velocity fields according to \eqref{eq:navigation velocity} which are used to update the humans' positions.
This setup exactly satisfies the theoretical assumptions.

The second simulation is based on the Unity game engine, where all humans and robots occupy certain small areas.
Collision avoidance is automatic.
The robots' motions are still simulated according to \eqref{eq:robot model}, but the humans continuously make decisions based on their local views; see Fig. \ref{fig:human}.
Specifically, a person moves according to the robot's instruction if a robot appears in his/her local view, and moves randomly otherwise.
The control algorithms are implemented in MATLAB.
An interface between Unity and MATLAB is created such that MATLAB extracts the states of all humans and robots from Unity, computes the robot controllers, and sends those commands back to Unity in real time.
In summary, this setup does not exactly match the theoretical assumptions.
It is used to evaluate whether the evacuation algorithm possesses a certain level of robustness to real-world scenarios.

In both simulations, 200 humans and 16 robots are simulated in a square workspace.
Human positions are randomly initialized according to a uniform distribution.
The robots are regularly initialized to ensure sufficient coverage of the environment at the beginning, but the orientations of their signs are randomly assigned.
The objective is to drive all humans to the upper right corner.
Hence, the target density is designed to be a ``narrow'' Gaussian centered at that location; see Fig. \ref{fig:target density and error}.
The interactive potential $W$ is set to 0 for simplicity.
In the Unity environment, the map size is $150\times150$ meters and the view range of a human is 30 meters in radius.
We discretize the domain so that all functions (like $\rho$) are approximated by a $30\times30$ matrix.
We compute the robot controls according to \eqref{eq:u and omega} where $v_d$ is given by \eqref{eq:vd} and $\rho$ is estimated using KDE.
The involved gradient operation is approximated using the finite difference method.
We note that the PDE system \eqref{eq:density dynamics} is only used for control design and stability analysis, and we never need to solve \eqref{eq:density dynamics}.

The simulation results are given in Fig. \ref{fig:evacuation Matlab} and \ref{fig:evacuation Unity}.
It is observed that with our feedback control laws, the robots dynamically adjust their positions and arrow signs based on the real-time crowd state.
In the MATLAB simulation, the humans perfectly follow the generated navigation velocity fields and are guided to the target location.
The density convergence error is plotted in Fig. \ref{fig:target density and error} and verifies the success of the evacuation objective.
In the Unity simulation, the humans are guided only when they see a robot near them and eventually they are also directed to the target location, which suggests that the evacuation algorithm is robust to scenarios that are not precisely modeled.

\begin{figure}[hbt!]
    \centering
        \includegraphics[width=0.2\textwidth]{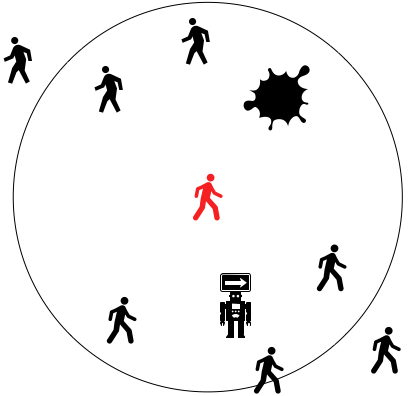}
    \caption{Local view of a human}
    \label{fig:human}
\end{figure}

\begin{figure}[hbt!]
    \centering
    \begin{subfigure}[b]{0.23\textwidth}
        \centering
        \includegraphics[width=\textwidth]{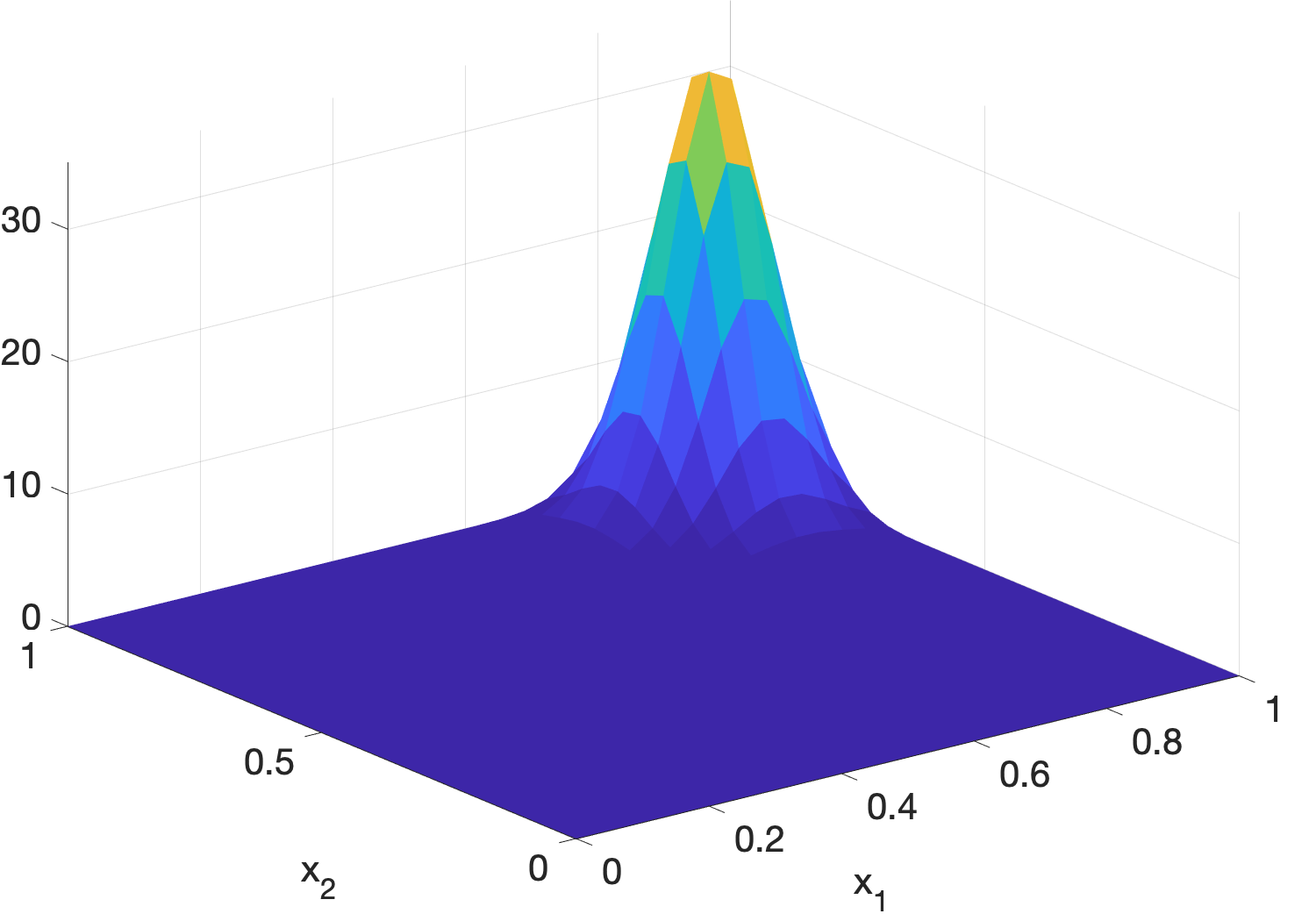}
    \end{subfigure}
    \begin{subfigure}[b]{0.23\textwidth}
        \centering
        \includegraphics[width=\textwidth]{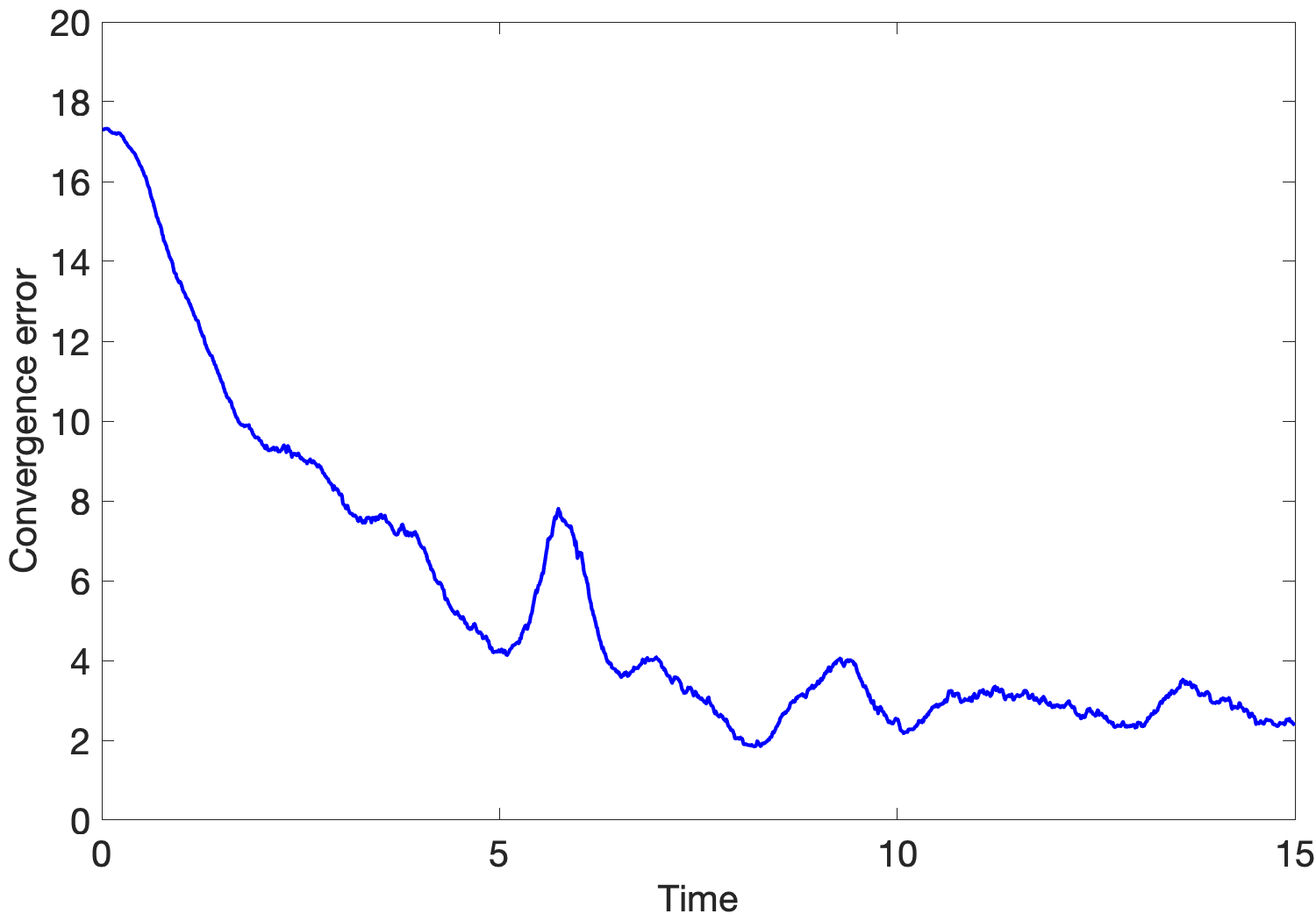}
    \end{subfigure}
    \caption{Left: target density. Right: convergence error.}
    \label{fig:target density and error}
\end{figure}

\begin{figure*}[t]
    \centering
    \begin{subfigure}[b]{0.22\textwidth}
        \centering
        \includegraphics[width=\textwidth]{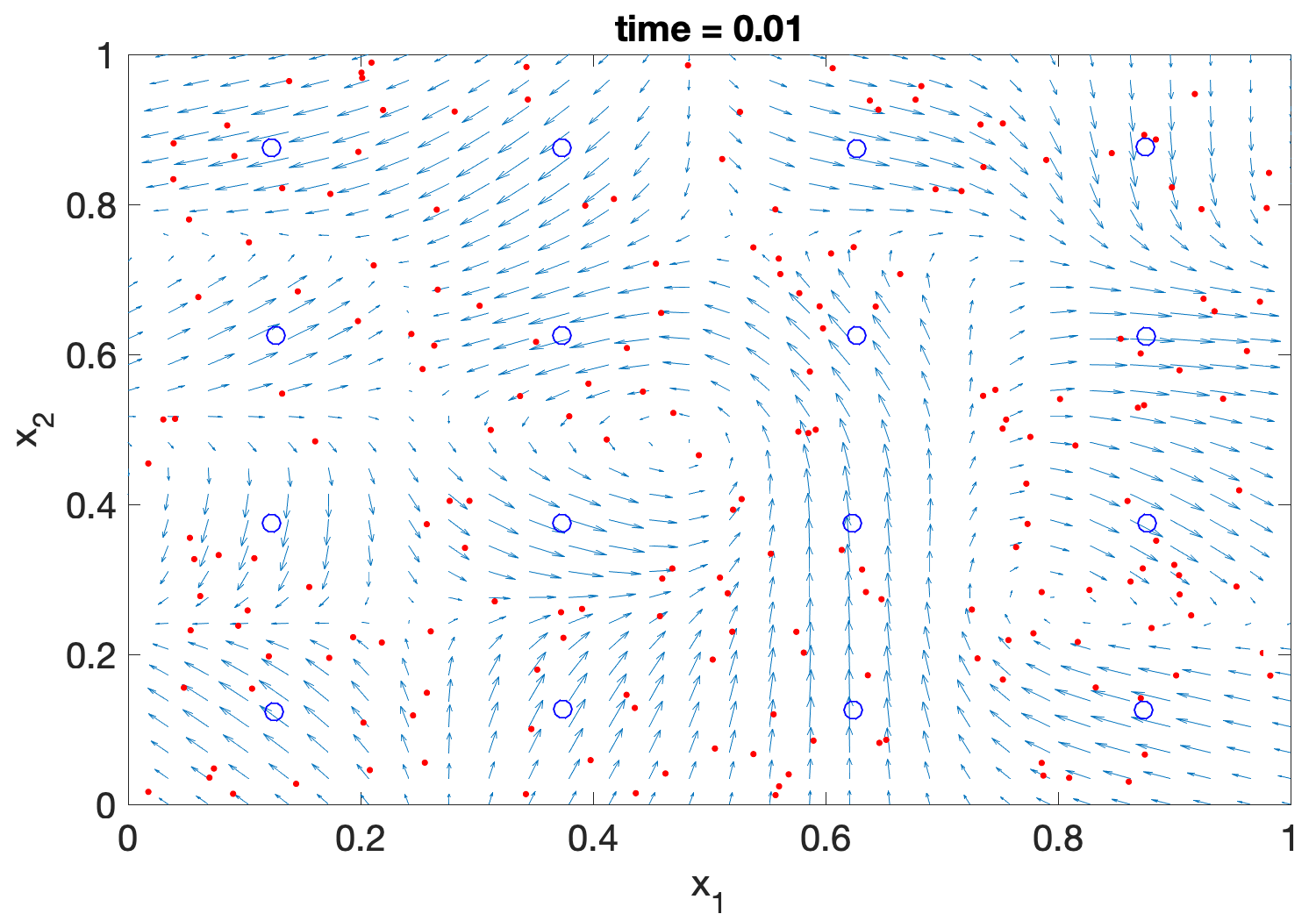}
    \end{subfigure}
    \begin{subfigure}[b]{0.22\textwidth}
        \centering
        \includegraphics[width=\textwidth]{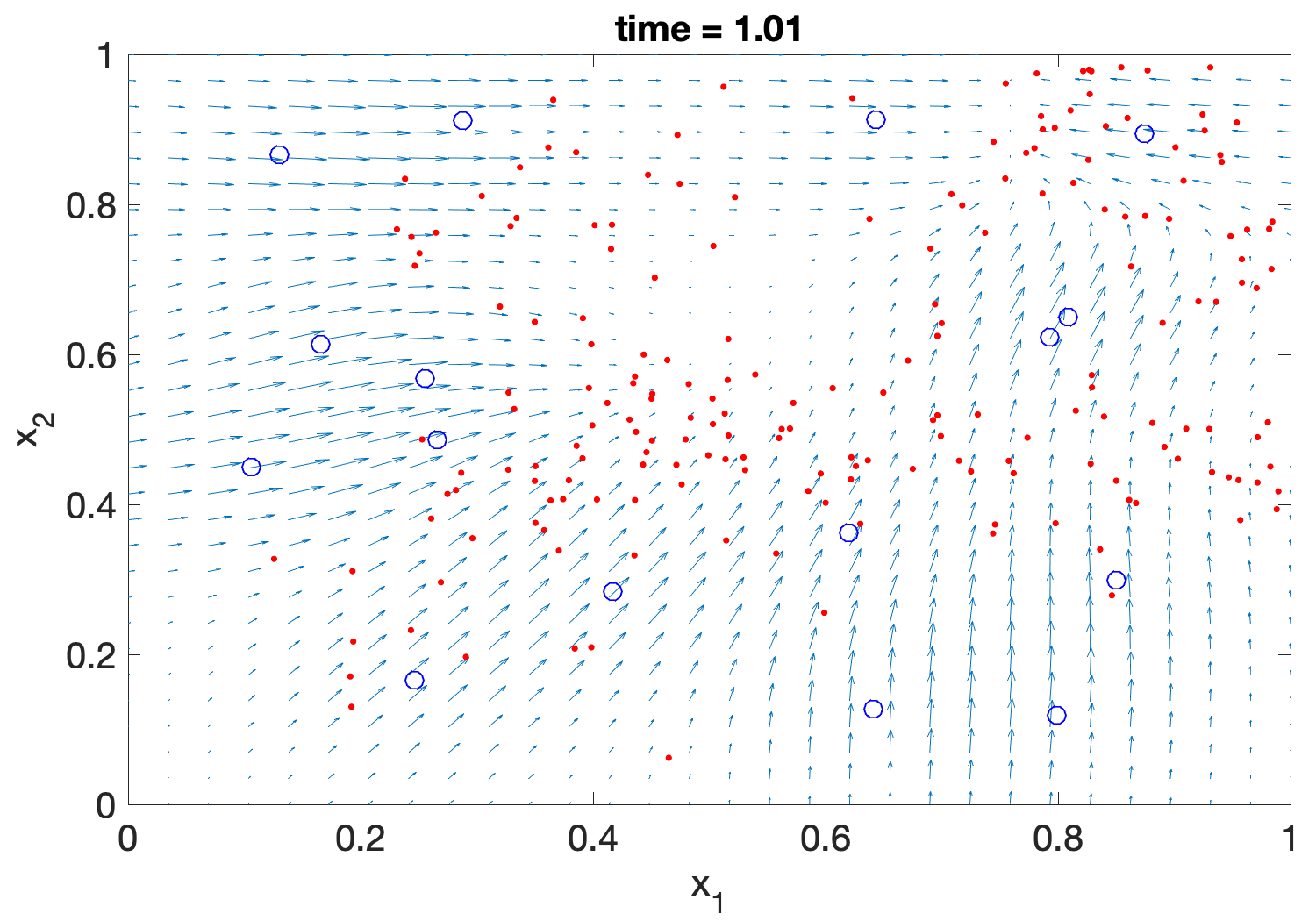}
    \end{subfigure}
    \begin{subfigure}[b]{0.22\textwidth}
        \centering
        \includegraphics[width=\textwidth]{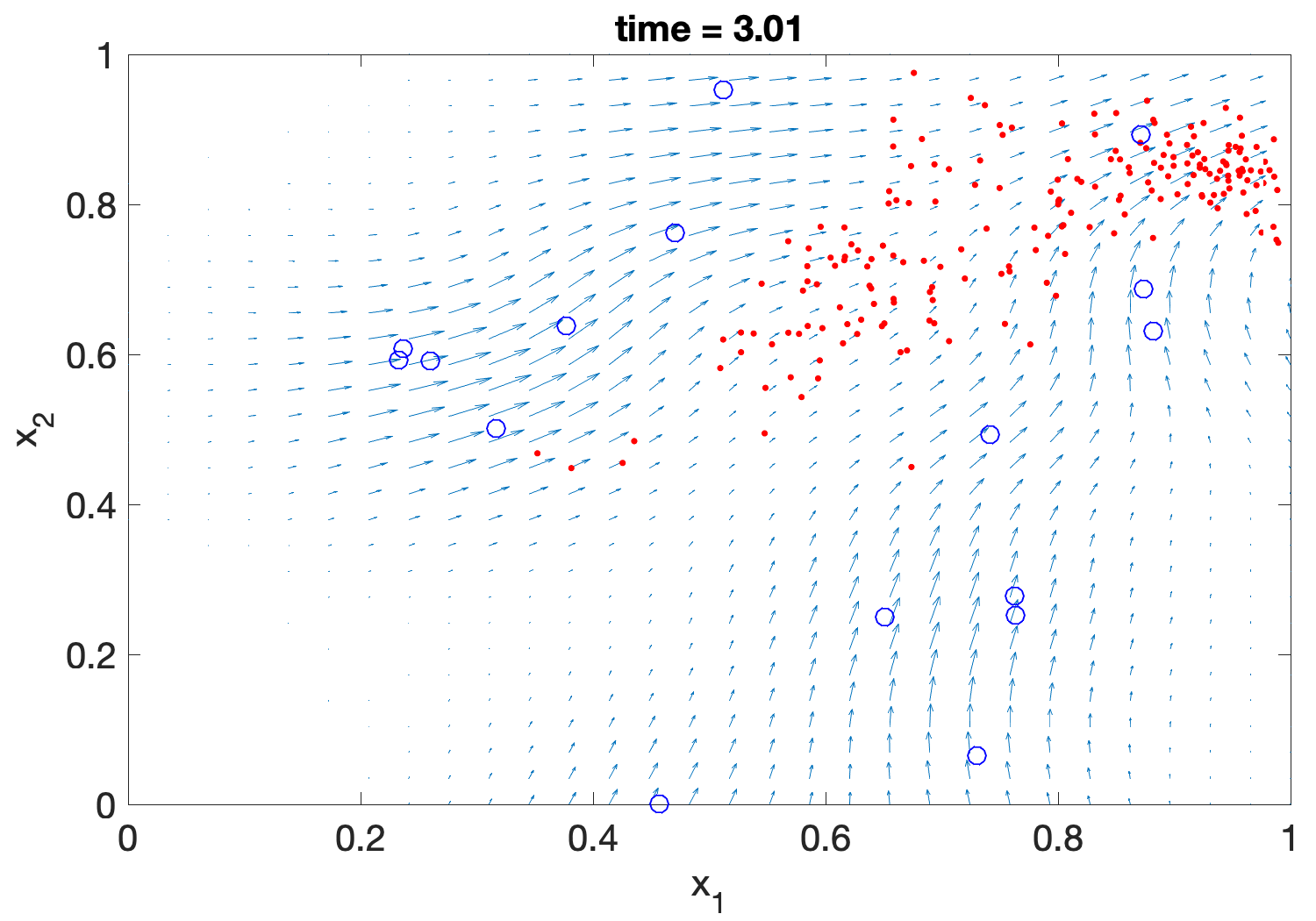}
    \end{subfigure}
    \begin{subfigure}[b]{0.22\textwidth}
        \centering
        \includegraphics[width=\textwidth]{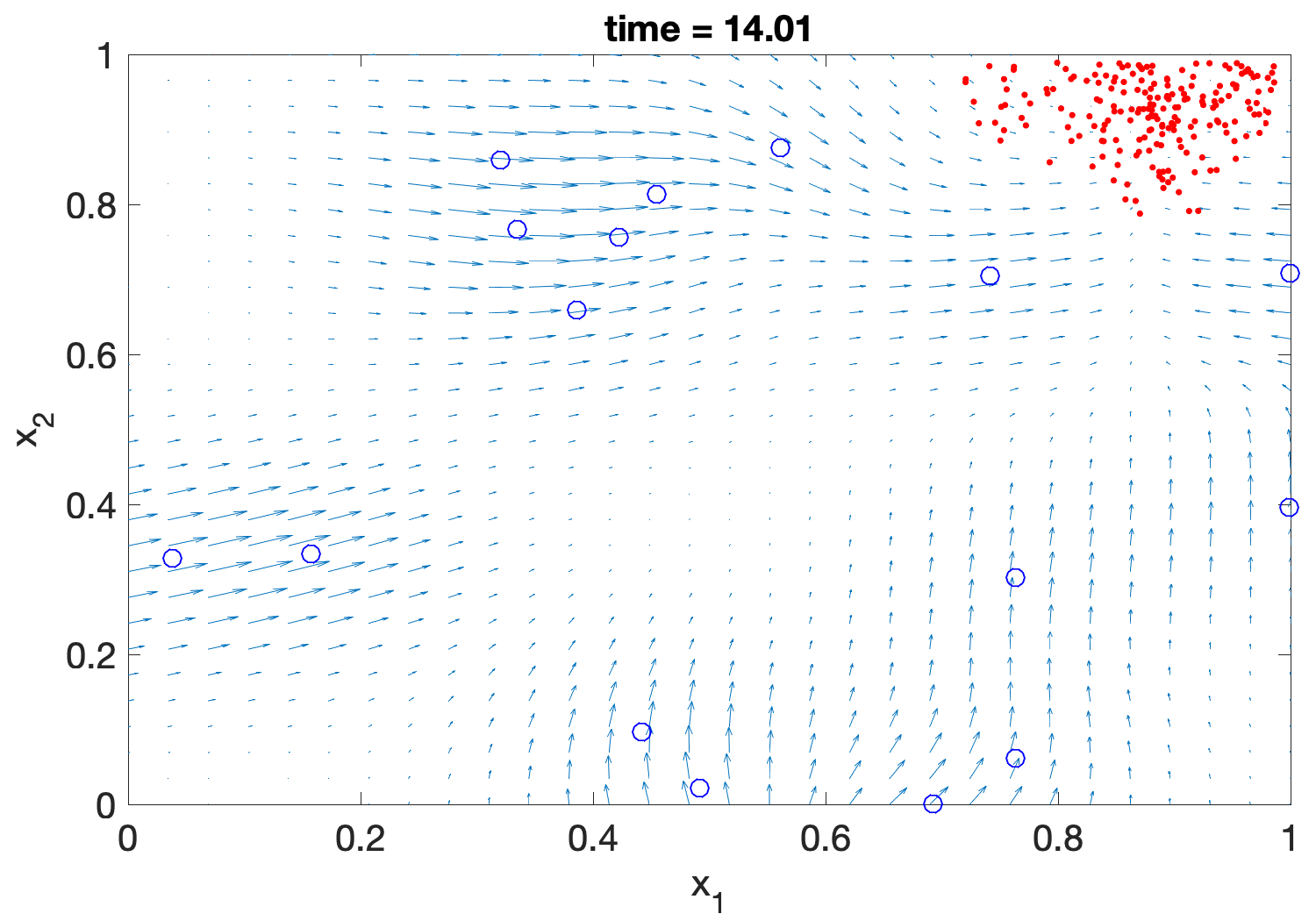}
    \end{subfigure}
    
    \begin{subfigure}[b]{0.22\textwidth}
        \centering
        \includegraphics[width=\textwidth]{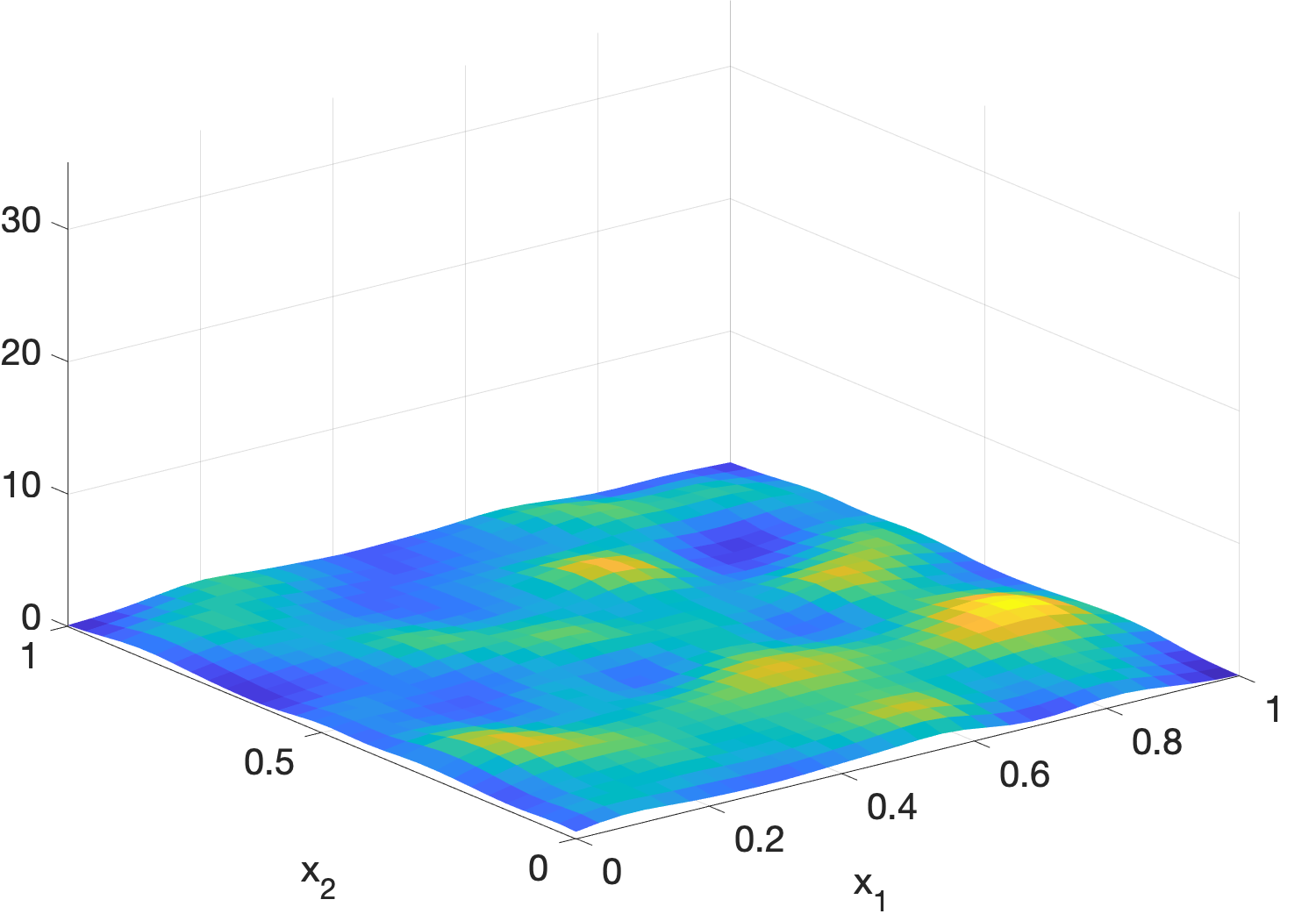}
    \end{subfigure}
    \begin{subfigure}[b]{0.22\textwidth}
        \centering
        \includegraphics[width=\textwidth]{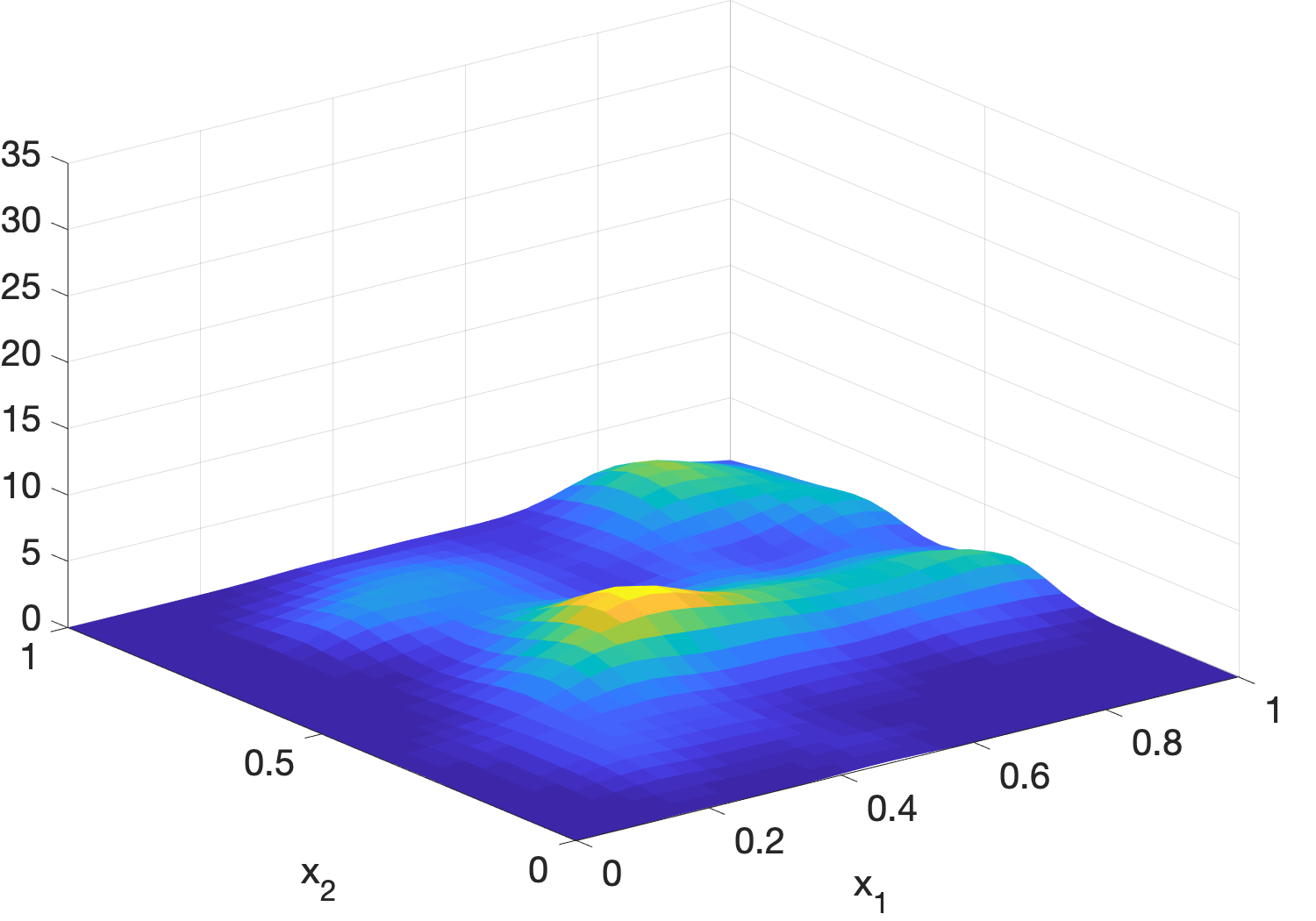}
    \end{subfigure}
    \begin{subfigure}[b]{0.22\textwidth}
        \centering
        \includegraphics[width=\textwidth]{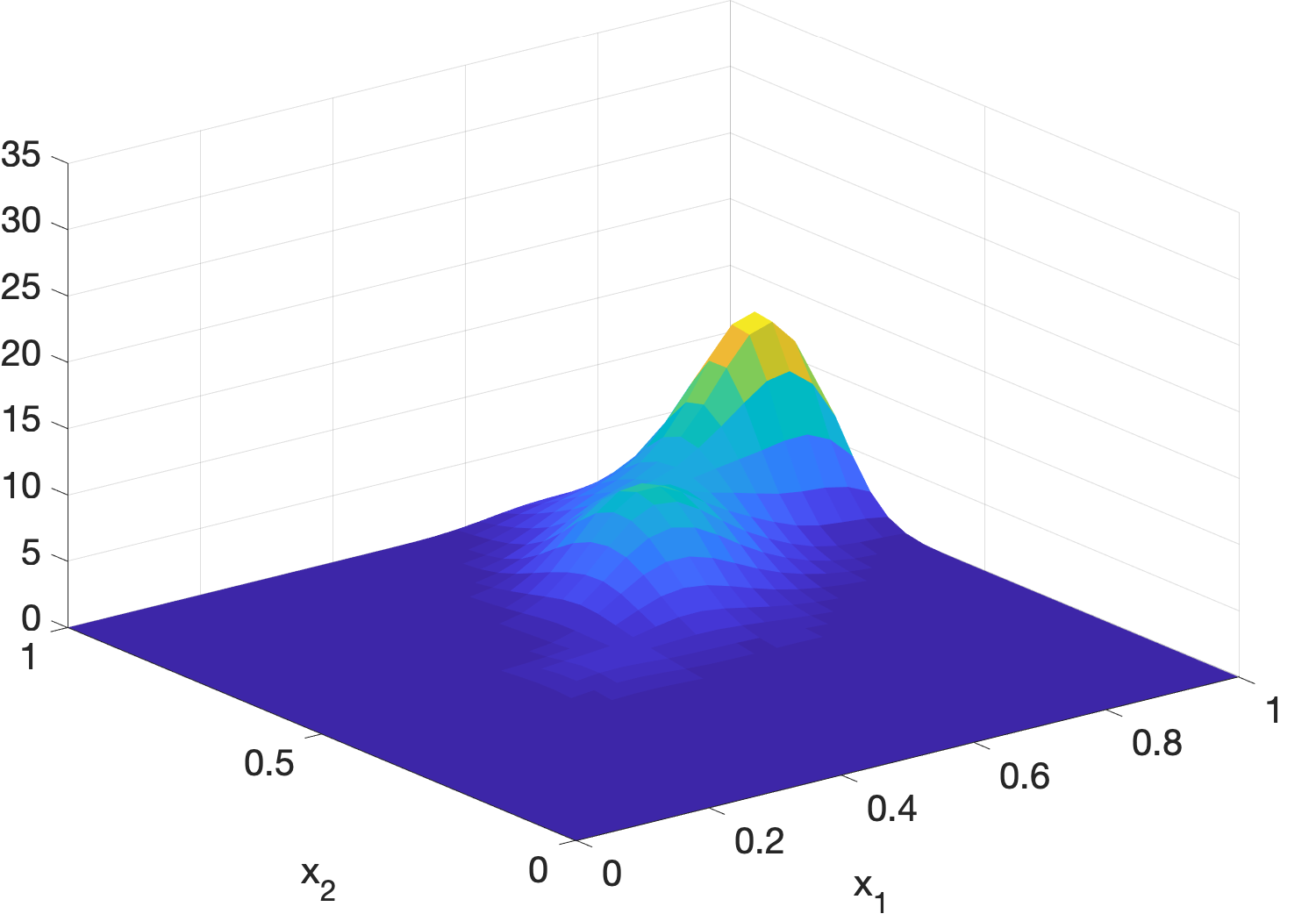}
    \end{subfigure}
    \begin{subfigure}[b]{0.22\textwidth}
        \centering
        \includegraphics[width=\textwidth]{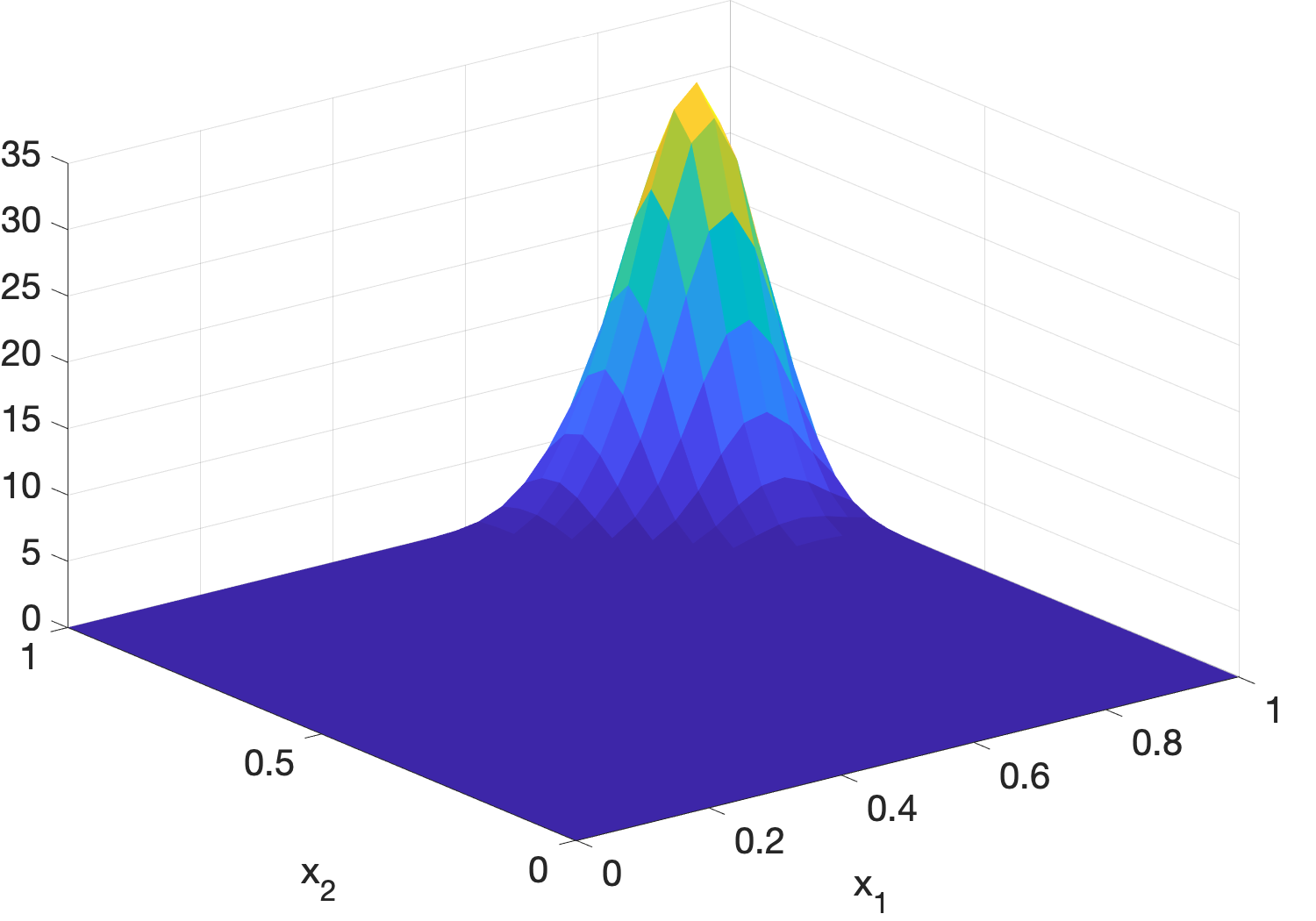}
    \end{subfigure}
    \caption{Simulated evacuation process in MATLAB. In the first row, the red dots represent the positions of humans. The blue circles represent the positions of robots. The arrows represent the navigation velocity field generated by the robots. The second row represents estimates of the current crowd density.}
    \label{fig:evacuation Matlab}
\end{figure*}

\begin{figure*}[t]
    \centering
    \begin{subfigure}[b]{0.22\textwidth}
        \centering
        \includegraphics[width=\textwidth]{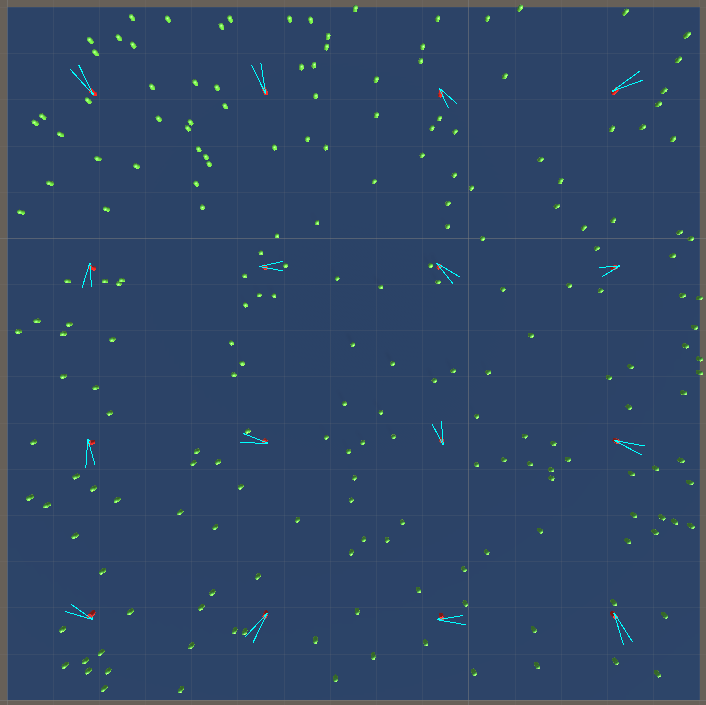}
    \end{subfigure}
    \begin{subfigure}[b]{0.22\textwidth}
        \centering
        \includegraphics[width=\textwidth]{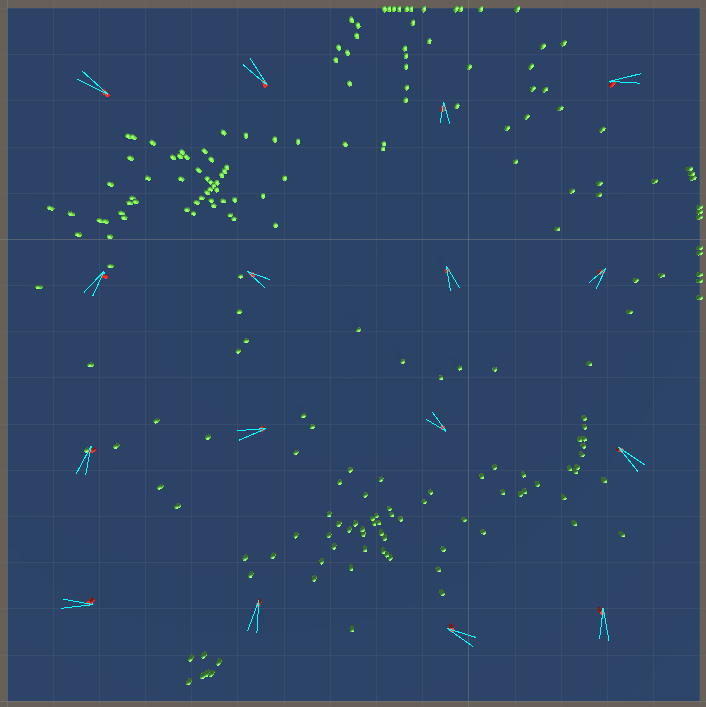}
    \end{subfigure}
    \begin{subfigure}[b]{0.22\textwidth}
        \centering
        \includegraphics[width=\textwidth]{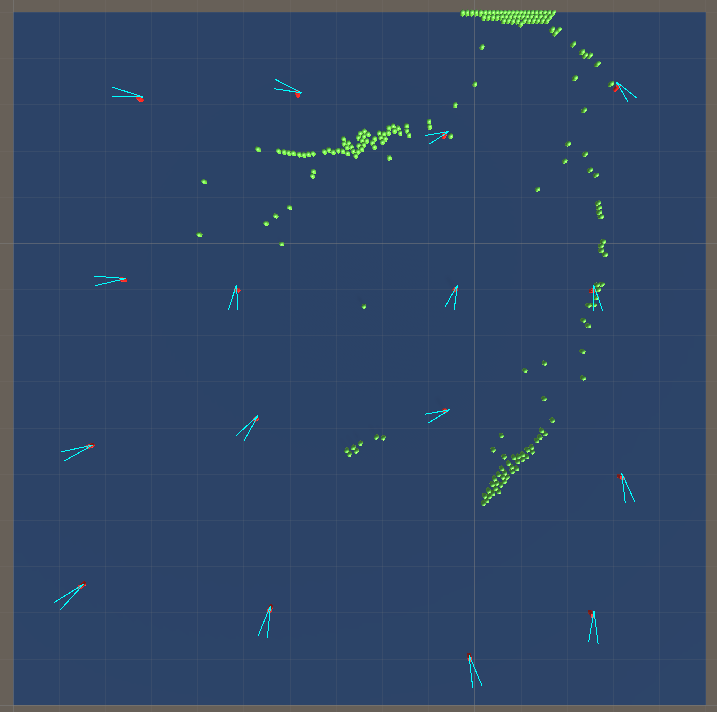}
    \end{subfigure}
    \begin{subfigure}[b]{0.22\textwidth}
        \centering
        \includegraphics[width=\textwidth]{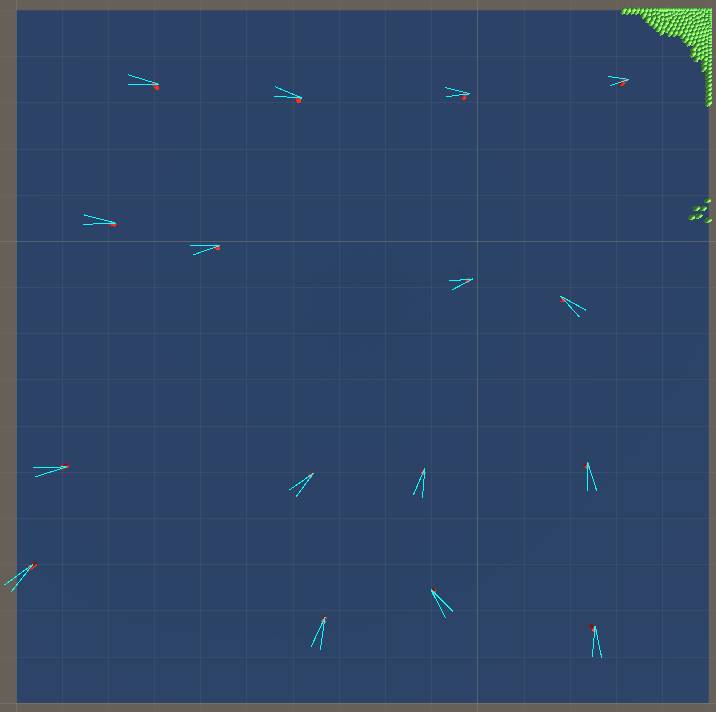}
    \end{subfigure}
    \caption{Simulated evacuation process in Unity. The green dots represent humans. The red dots represent robots. The direction of a robot's arrow sign is indicated by the $">"$ sign.}
    \label{fig:evacuation Unity}
\end{figure*}

\section{Conclusion}
\label{section:conclusion}
In this work, we studied a multi-robot-assisted crowd evacuation problem when humans significantly outnumber the robots.
We addressed the challenge by explicitly integrating human-robot interactions into the mean-field models and formulating the crowd evacuation problem as a density control problem.
Then, we designed density feedback control laws for the robots to dynamically adjust their positions such that their generated navigation velocity fields stabilize the crowd density, and proved the stability of the algorithms.
Our future work is to study the influence of the number of robots and decentralize the control algorithm.

\bibliographystyle{IEEEtran}
\bibliography{Refs_Tongjia,Refs_Jack,Refs_Mollik}

\end{document}